\documentclass[wcp]{jmlr}

\usepackage{amsmath,amssymb,graphicx,url,doi}
\usepackage{setspace,booktabs,multirow}

\title{Skew-adaptive conformal prediction}

\author{
\Name{Paulo C. Marques} F.
\Email{paulocmf1@insper.edu.br} \\
\addr{Insper Institute of Education and Research \\ São Paulo, SP, 04546-042, Brazil}
\AND
\Name{Helton Graziadei}
\Email{helton@ufscar.br} \\
\addr{Department of Statistics \\ Federal University of São Carlos \\ São Carlos, SP, 13565-905, Brazil}
}

\jmlrproceedings{}{}
\jmlrvolume{}
\jmlrpages{}
\jmlryear{}
\jmlrworkshop{}

\begin{document}

\maketitle

\begin{abstract}
We develop a skew-adaptive extension of split conformal prediction for regression. The method starts from an asymmetric interval family centered at a point prediction and uses the gauge approach to deduce the conformity score induced by this family. The inverse hyperbolic sine transform of signed scaled residuals provides the training target for an additional predictive model, whose role is to learn how predictive uncertainty should tilt across the feature space. The resulting procedure preserves the finite-sample marginal validity of split conformal prediction under exchangeability, while producing intervals that adapt to both local scale and local skewness. We also develop a calibration-sample-based estimator for comparing the expected relative future width of the skew-adaptive and classical scaled-score intervals. Experiments on a variety of datasets indicate gains in prediction interval efficiency over the scaled-score construction and conformalized quantile regression, and show that the proposed estimator closely matches the corresponding average width ratio observed on the test sample.
\end{abstract}


\section{Introduction}\label{sec:intro}

Contemporary predictive methods, ranging from modern regression procedures to the artifacts of artificial intelligence that now shape scientific, industrial, and social life, increasingly derive their strength not from tightly prescribed modeling specifications, but from the ability to learn complex associations directly from large amounts of data. This shift toward machine-learning-based systems in an age of data abundance has produced remarkable gains in predictive and inferential capabilities, while at the same time making it substantially more challenging to assess the uncertainty surrounding the resulting predictions.

When the predictive mechanism is highly flexible and only weakly tied to a prescribed generative model, uncertainty assessments that rely on strong assumptions about either the data-generating process or the structure of the predictive model become difficult to justify. This creates a demand for methods that can attach theoretically guaranteed measures of uncertainty to predictions while relying on assumptions as weak as possible. Conformal prediction \citep{vovk2005} meets this demand by allowing the construction of prediction sets with finite-sample validity under the sole assumption of data exchangeability, regardless of the predictive model under consideration.

In this work, we focus on regression problems within inductive, or split, conformal prediction. Our goal is to develop a natural extension of the scaled-score method introduced by \cite{papadopoulos2002}, moving beyond prediction intervals that expand symmetrically around a central point prediction. More specifically, we seek a conformal procedure that produces asymmetric prediction intervals around the point prediction, allowing the interval to adapt not only to the local scale of predictive uncertainty, but also to its local asymmetry.

We proceed as follows. In Section \ref{sec:canonical}, we briefly recall the standard description of split conformal prediction. In Section \ref{sec:dual}, we discuss the dual view of conformity scores developed by \cite{gupta2022}, in which a gauge function allows the appropriate conformity score to be deduced from the form of a prescribed family of prediction intervals. In Section \ref{sec:family}, we introduce a fairly general family of asymmetric prediction intervals and derive its associated conformity score through the gauge method. Section \ref{sec:gamma} develops a sequential procedure that naturally extends the scaled-score method of \cite{papadopoulos2002}. After learning a central prediction and a local scale, the inverse-hyperbolic-sine transform of the signed scaled residuals is used as the training target for a third predictive model, whose goal is to learn the local asymmetry of the prediction task. This leads to a formal statement of the resulting skew-adaptive conformal prediction procedure. Section \ref{sec:efficiency} discusses issues of prediction interval efficiency.
Section \ref{sec:experiments} presents empirical results on several datasets. In addition to the classical scaled-score method of \cite{papadopoulos2002}, we compare our proposal with the conformalized quantile regression method of \cite{romano2019}, obtaining favorable results for the skew-adaptive approach. Finally, in Section \ref{sec:conclusion}, we provide pointers to open source code and briefly comment on possible uses of the proposed construction beyond the inductive setting considered here.

\section{Split conformal prediction the canonical way}\label{sec:canonical}

Let $T=\{(X'_1,Y'_1),\dots,(X'_m,Y'_m)\}$ be a training sample of random pairs, with feature vectors $X'_i\in\mathbb{R}^d$ and response variables $Y'_i\in\mathbb{R}$, and introduce an exchangeable sequence of random pairs
\[
  \underbrace{(X_1,Y_1),\dots,(X_n,Y_n)}_{\text{calibration}},\underbrace{(X_{n+1},Y_{n+1}),\dots}_{\text{future}}
\]
in which, similarly, $X_i\in\mathbb{R}^d$ and $Y_i\in\mathbb{R}$. The first $n$ pairs in this exchangeable sequence constitute our calibration sample, followed by the future observable pairs. In practice, the training and calibration samples are obtained by randomly splitting the available data; hence the term split conformal prediction.

From a realization of the training set $T$, we construct a \textit{conformity function} $\rho:\mathbb{R}^d\times\mathbb{R}\to\mathbb{R}$ whose general purpose is to contrast, within the exchangeable data sequence, inferences made from feature vectors with values of the associated response variables. For example, in \cite{papadopoulos2002}, one first learns a regression function $\hat\mu:\mathbb{R}^d\to\mathbb{R}$ from a realization of $T$. Subsequently, a second positive predictive model $\hat\sigma:\mathbb{R}^d\to\mathbb{R}^+$ is trained from a realization of the sample $(X'_1,\Delta_1),\dots,(X'_m,\Delta_m)$, in which $\Delta_i=|Y'_i-\hat{\mu}(X'_i)|$, for $i=1,\dots,m$. The intuition is that $\hat{\sigma}$ learns how large the absolute prediction error of $\hat{\mu}$ is expected to be for a given $x\in\mathbb{R}^d$, thereby measuring local changes in prediction difficulty across the feature space. After the models $\hat{\mu}$ and $\hat{\sigma}$ have been learned, the conformity function is defined as the scaled absolute residual:  $\rho(x,y)=|y-\hat{\mu}(x)|/\hat{\sigma}(x)$.

Let $\lceil t\rceil=\min\{k\in\mathbb{Z}:t\le k\}$ denote the ceiling of $t\in\mathbb{R}$. The following universal property, proved in \cite{papadopoulos2002}, is the distinctive feature of conformal methods.

\begin{proposition}\label{prop:mvp}
Define the conformity scores $R_i=\rho(X_i,Y_i)$, for $i\ge 1$, and choose a target nominal miscoverage level $0<\alpha<1$ such that $\left\lceil (1-\alpha)(n+1)\right\rceil \leq n$. Denote the order statistics of the calibration scores $\{R_1,R_2,\dots,R_n\}$ by $R_{(1)}\leq R_{(2)}\le\dots\le R_{(n)}$. Introducing the notation $\hat{r} = R_{\left(\lceil (1-\alpha)(n+1)\rceil\right)}$, it follows that
\begin{equation}\label{eq:mvp}
  \Pr(Y_{n+1}\in C(X_{n+1})) \ge 1 - \alpha,
\end{equation}
in which  $C(x)=\{y\in\mathbb{R}:\rho(x,y)\le \hat r\}$, with $x\in\mathbb{R}^d$.
\end{proposition}

\begin{proof}
The assumption of an exchangeable data sequence implies that the full set of conformity scores $\{R_1,R_2,\dots,R_n,R_{n+1}\}$ is exchangeable. We denote the order statistics of this full set by $\tilde R_{(1)}\leq \tilde R_{(2)}\leq\dots\leq \tilde R_{(n)}\leq \tilde R_{(n+1)}$. Taking into account that we may have ties, it follows from the definition of $\tilde R_{(k)}$ that at least $k$ of the $R_i$'s in the full set are less than or equal to $\tilde R_{(k)}$, for $k=1,\dots,n+1$. Hence, ${\sum_{i=1}^{n+1} I_{\{R_i\leq \tilde R_{(k)}\}}\geq k}$, almost surely. Taking expectations in the last inequality and noting that, by exchangeability, the probability $\Pr(R_i\leq \tilde R_{(k)})$ is the same for ${i=1,\dots,n+1}$, we conclude that ${\Pr(R_{n+1}\leq \tilde R_{(k)})\geq k/(n+1)}$. Furthermore, for $k=1,\dots,n$, notice that $R_{n+1}>R_{(k)}$ if and only if $R_{n+1}>\tilde R_{(k)}$, because $R_{n+1}$ cannot be strictly larger than itself. Hence, ${\Pr(R_{n+1}\leq R_{(k)})=\Pr(R_{n+1}\leq \tilde R_{(k)})}$, for $k=1,\dots,n$. The choice $k=\left\lceil (1-\alpha)(n+1)\right\rceil$ yields $\Pr(R_{n+1}\le R_{\left(\left\lceil (1-\alpha)(n+1)\right\rceil\right)})\ge 1-\alpha$, and the result follows from the remaining definitions in the proposition statement.
\end{proof}

Property (\ref{eq:mvp}) is referred to as the marginal validity of conformal prediction intervals. In split conformal prediction, this fundamental property is accompanied by further analytical information. For a finite batch of future observables, its empirical coverage admits an exact distributional characterization. This characterization yields, in particular, a practical procedure for determining the calibration sample size required in applications. It also allows the marginal validity property (\ref{eq:mvp}) to be viewed as an inequality constraint on the expected empirical coverage of the future batch; see \cite{marques2025a} and references therein for details.

It follows from Proposition \ref{prop:mvp} that the scaled absolute residual conformity function $\rho(x,y)=|y-\hat{\mu}(x)|/\hat{\sigma}(x)$ used in \cite{papadopoulos2002} leads to a conformal prediction interval of the form
\begin{equation}\label{eq:scaledinterval}
  C(x)=\bigl[\hat{\mu}(x)-\hat{r}\,\hat{\sigma}(x),\ \hat{\mu}(x)+\hat{r}\,\hat{\sigma}(x)\bigr].
\end{equation}
In what follows, we will refer to this construction as the scaled-score method. The resulting interval is symmetric around the central prediction $\hat\mu(x)$. Although the model $\hat\sigma$ allows the interval length to vary with the feature vector, the expansion still occurs by the same amount to the left and to the right of $\hat\mu(x)$. The natural question, then, is whether this scale-adaptive construction can be extended so that the two sides of the interval are allowed to adapt separately, producing conformal prediction intervals that remain data-driven and marginally valid, but are no longer constrained to be symmetric around the central prediction $\hat\mu(x)$.

\section{A dual view of conformity scores}\label{sec:dual}

According to \cite{vanvleck1916}, Jacobi is said to have inculcated in his students the \textit{dictum} ``Man muss immer umkehren'', that is, one must always invert, or always seek a converse. Taken as a heuristic principle, this suggests examining a familiar construction and turning it inside out, reformulating the problem in the reverse direction so as to uncover a fertile mathematical structure. In Jacobi's spirit, \cite{gupta2022} developed an inversion leading to a dual formulation of the conformity scores discussed in Section \ref{sec:canonical}, a perspective that will be essential in our development of the skew-adaptive conformal prediction procedure in Sections \ref{sec:family} and \ref{sec:gamma}.

We will use the method of \cite{gupta2022} in the following form. For $x\in\mathbb{R}^d$, consider a parameterized family $\{C_r(x)\}_{r\ge 0}$ of nonempty closed intervals of $\mathbb{R}$, nondecreasing in the sense that $C_r(x)\subseteq C_t(x)$ whenever $r\le t$, and right-continuous in the real parameter $r$, in the sense that
\[
  C_r(x)=\bigcap_{t>r} C_t(x).
\]
We refer to the nonnegative real parameter $r$ indexing $C_r(x)$ as the interval expansion factor.

Define the \textit{gauge function} associated with the family $\{C_r(x)\}_{r\ge 0}$ by
\[
  s(x,y)=\inf\{r\ge 0:y\in C_r(x)\},
\]
for $x\in\mathbb{R}^d$ and $y\in\mathbb{R}$. Intuitively, for any given feature vector $x$, the gauge $s(x,y)$ measures the smallest expansion factor $r$ at which the corresponding interval $C_r(x)$ contains the response value $y$. In other words, $s$ gauges how much the interval family must expand before it first reaches the response. The following simple fact builds a bridge with the canonical formulation of split conformal prediction discussed in Section \ref{sec:canonical}.

\begin{proposition}\label{prop:thebridge}
For every $r\ge 0$ and each $x\in\mathbb{R}^d$, we have $C_r(x)=\{y\in\mathbb{R}:s(x,y)\le r\}$.
\end{proposition}

\begin{proof}
Fix $r\ge 0$ and $x\in\mathbb{R}^d$. We first prove that $C_r(x)\subseteq\{y\in\mathbb{R}:s(x,y)\le r\}$. Let $y\in C_r(x)$ and define
$A_{x,y}=\{u\ge 0:y\in C_u(x)\}$, so that $s(x,y)=\inf A_{x,y}$. Since $y\in C_r(x)$, we have $r\in A_{x,y}$. It follows immediately that
$s(x,y)=\inf A_{x,y}\le r$. Hence $y\in\{z\in\mathbb{R}:s(x,z)\le r\}$, and therefore $C_r(x)\subseteq\{y\in\mathbb{R}:s(x,y)\le r\}$. Conversely, suppose that $y\in\mathbb{R}$ satisfies $s(x,y)\le r$. We will show that $y$ belongs to every $C_t(x)$ with $t>r$. To this end, fix an arbitrary $t>r$. Since $s(x,y)\le r<t$, we have $\inf A_{x,y}<t$. By the defining property of the infimum, there must exist some $v\in A_{x,y}$ such that $v<t$. Indeed, if no such $v$ existed, then every $u\in A_{x,y}$ would satisfy $u\ge t$, so that $t$ would be a lower bound for $A_{x,y}$, forcing $\inf A_{x,y}\ge t$, a contradiction. Since $v\in A_{x,y}$, we have $y\in C_v(x)$. Since $v<t$ and the family is nondecreasing, $C_v(x)\subseteq C_t(x)$, and therefore $y\in C_t(x)$. As this argument holds for every $t>r$, we obtain $y\in\bigcap_{t>r}C_t(x)$. By the assumed right-continuity of the family, $\bigcap_{t>r}C_t(x)=C_r(x)$. Hence $y\in C_r(x)$, proving that
$\{y\in\mathbb{R}:s(x,y)\le r\}\subseteq C_r(x)$. The two inclusions give the claim.
\end{proof}

\begin{figure}[t!]
\centering
\includegraphics[width=15cm]{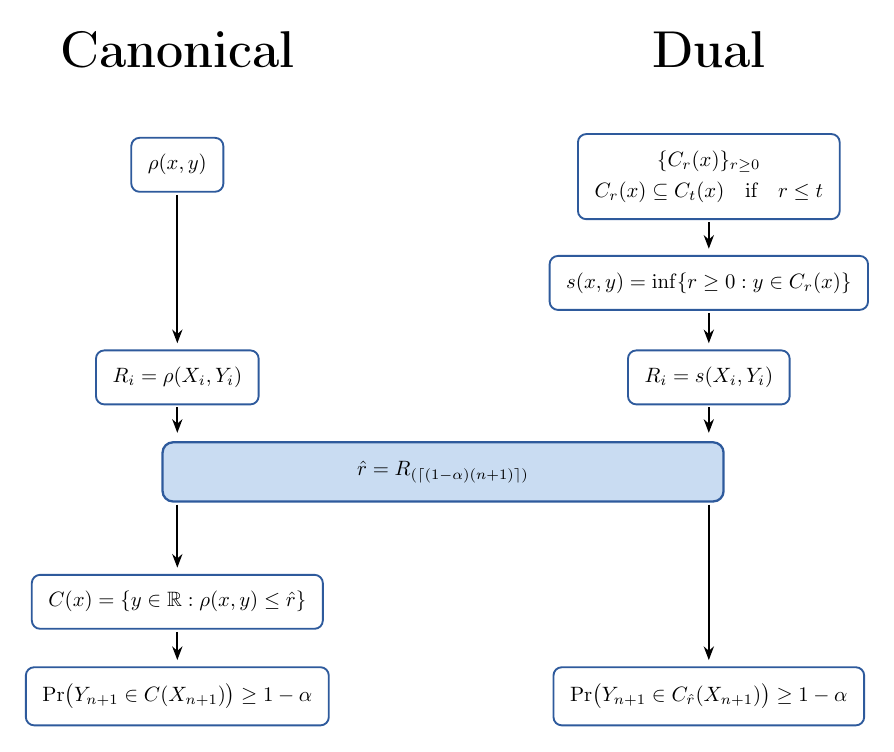} 
\caption{Canonical and dual views of split conformal prediction. The canonical view starts from a specified conformity function. This function is evaluated on each calibration pair to produce conformity scores, and the resulting calibrated threshold defines the prediction interval as a sublevel set of the conformity function. The dual view starts from the opposite end: it first specifies a nondecreasing family of prediction intervals, and then uses its associated gauge as the conformity function to which the split conformal calibration process is applied. The two views therefore share the same finite-sample validity property, but place the design choices at different points in the construction. In the canonical view, the geometry of the prediction set is determined by the chosen conformity function. In the dual view, the conformity score is induced by the prescribed geometry of the interval family.}
\label{fig:canonicalvsdual}
\end{figure}

The following proposition shows that, once the family $\{C_r(x)\}_{r\ge 0}$ has been specified, its gauge function $s$ can be used as an ordinary conformity function. In this way, every interval in the family whose expansion factor is no smaller than a suitable calibrated threshold satisfies a marginal validity property for the future pair $(X_{n+1},Y_{n+1})$. We follow the definitions and notations in Proposition \ref{prop:mvp}.

\begin{proposition}\label{prop:mvpdual}
For a given family $\{C_r(x)\}_{r\ge 0}$ and associated gauge function $s$, define the calibration conformity scores by $R_i=s(X_i,Y_i)$, for $i=1,\dots,n$. It follows that
\[
  \Pr(Y_{n+1}\in C_{\hat r}(X_{n+1}))\ge 1-\alpha,
\]
in which $\hat r=R_{\left(\lceil(1-\alpha)(n+1)\rceil\right)}$.
\end{proposition}

\begin{proof}
The closing argument in the proof of Proposition \ref{prop:mvp} shows that \[ \Pr(s(X_{n+1},Y_{n+1})\le \hat r)\ge 1-\alpha. \]
By Proposition \ref{prop:thebridge}, we have that $C_r(x)=\{y\in\mathbb{R}:s(x,y)\le r\}$. Therefore, the event ${\{s(X_{n+1},Y_{n+1})\le \hat r\}}$ coincides with the event $\{Y_{n+1}\in C_{\hat r}(X_{n+1})\}$, and the claim follows.
\end{proof}

Figure \ref{fig:canonicalvsdual} summarizes the preceding discussion, emphasizing the central inversion behind the dual view: in the canonical construction, the prediction set follows from the chosen conformity function, whereas in the dual construction, the conformity score follows from the prescribed geometry of the interval family through its gauge.

\section{A skewed family of prediction intervals}\label{sec:family}

We now introduce the family of prediction intervals that will underlie the skew-adaptive procedure. Let $\hat\mu:\mathbb{R}^d\to\mathbb{R}$ be a central prediction function, and let $\hat a:\mathbb{R}^d\to\mathbb{R}^+$ and $\hat b:\mathbb{R}^d\to\mathbb{R}^+$ be two positive functions. For a feature vector $x\in\mathbb{R}^d$ and each expansion factor $r\ge 0$, consider the rather general interval
\begin{equation}\label{eq:general}
  C_r(x) = \bigl[ \hat\mu(x)-r\,\hat a(x), \hat\mu(x)+r\,\hat b(x) \bigr].
\end{equation}
The function $\hat a$ controls the expansion of the interval to the left of $\hat\mu(x)$, while $\hat b$ controls the expansion to the right. The corresponding interval family $\{C_r(x)\}_{r\ge 0}$ is nondecreasing in $r$, since increasing the expansion factor moves both endpoints outward, and it is right-continuous in $r$ because the endpoints depend continuously on the expansion factor $r$.

The next step is to reparameterize this interval family so that the component controlling scale is separated from the component controlling asymmetry around the central prediction. Consider the bijective transformation $(\hat a(x), \hat b(x)) \mapsto (\hat\sigma(x), \hat\gamma(x))$ defined by
\[
  \hat\sigma(x)=\sqrt{\hat a(x)\hat b(x)}
  \qquad\text{and}\qquad
  \hat\gamma(x)=\frac{1}{2}\log\!\left(\frac{\hat b(x)}{\hat a(x)}\right),
\]
with inverse given by $\hat a(x)=\hat\sigma(x)e^{-\hat\gamma(x)}$ and $\hat b(x)=\hat\sigma(x)e^{\hat\gamma(x)}$. In these new coordinates, interval (\ref{eq:general}) takes the form
\[
  C_r(x)=
  \bigl[
    \hat\mu(x)-r\,\hat\sigma(x)e^{-\hat\gamma(x)},
    \hat\mu(x)+r\,\hat\sigma(x)e^{\hat\gamma(x)}
  \bigr].
\]

Since the transformation is bijective, this reparameterization does not restrict the expressiveness of the original interval family in any way. It merely rewrites the family in coordinates that separate local scale, represented by $\hat\sigma(x)$, from local asymmetry around $\hat\mu(x)$, represented by $\hat\gamma(x)$. A positive value of $\hat\gamma(x)$ expands the interval more to the right of $\hat\mu(x)$ than to the left, while a negative value produces the opposite effect. The gauge function for the corresponding interval family can be derived in analytic form.

\begin{proposition}\label{prop:skewgauge}
For the skewed interval family $\{C_r(x)\}_{r\ge 0}$ defined, for $x\in\mathbb{R}^d$, by
\begin{equation}\label{eq:asymm}
  C_r(x)
  =
  \left[
    \hat\mu(x)-r\,\hat\sigma(x)e^{-\hat\gamma(x)},
    \hat\mu(x)+r\,\hat\sigma(x)e^{\hat\gamma(x)}
  \right],
\end{equation}
in which $\hat\mu(x),\hat\gamma(x)\in\mathbb{R}$ and $\hat\sigma(x)>0$, the gauge function is given by
\[
  s(x,y)
  =
  \max\left\{
    \frac{(\hat\mu(x)-y)_+}{\hat\sigma(x)e^{-\hat\gamma(x)}},
    \frac{(y-\hat\mu(x))_+}{\hat\sigma(x)e^{\hat\gamma(x)}}
  \right\},
\]
in which $u_+=\max\{u,0\}$.
\end{proposition}

\begin{proof}
For any $x\in\mathbb{R}^d$ and $y\in\mathbb{R}$, since $\hat\sigma(x)e^{-\hat\gamma(x)}>0$ and $\hat\sigma(x)e^{\hat\gamma(x)}>0$, using the definitions we have
\[
\begin{aligned}
  s(x,y) &= \inf\{r\ge 0:y\in C_r(x)\} \\
  &= \inf\left\{r\ge 0:\hat\mu(x)-r\,\hat\sigma(x)e^{-\hat\gamma(x)}\le y\le \hat\mu(x)+r\,\hat\sigma(x)e^{\hat\gamma(x)}\right\} \\
  &= \inf\left\{r\ge 0:r\ge\frac{\hat\mu(x)-y}{\hat\sigma(x)e^{-\hat\gamma(x)}}\;\;\text{and}\;\;r\ge\frac{y-\hat\mu(x)}{\hat\sigma(x)e^{\hat\gamma(x)}}\right\}.
\end{aligned}
\]
Thus the admissible values of $r$ are precisely those nonnegative numbers that are at least as large as both lower bounds above. The smallest such $r$ is the maximum of the three lower bounds, namely
\[
  s(x,y)
  =
  \max\left\{
    \frac{\hat\mu(x)-y}{\hat\sigma(x)e^{-\hat\gamma(x)}},
    \frac{y-\hat\mu(x)}{\hat\sigma(x)e^{\hat\gamma(x)}},
    0
  \right\}.
\]
The previous expression can be rewritten with positive parts. Since both denominators are positive,
\[
  \max\left\{
    \frac{\hat\mu(x)-y}{\hat\sigma(x)e^{-\hat\gamma(x)}},
    0
  \right\}
  =
  \frac{(\hat\mu(x)-y)_+}{\hat\sigma(x)e^{-\hat\gamma(x)}}
\quad \text{and} \quad
  \max\left\{
    \frac{y-\hat\mu(x)}{\hat\sigma(x)e^{\hat\gamma(x)}},
    0
  \right\}
  =
  \frac{(y-\hat\mu(x))_+}{\hat\sigma(x)e^{\hat\gamma(x)}}.
\]
Consequently,
\[
  s(x,y)
  =
  \max\left\{
    \frac{(\hat\mu(x)-y)_+}{\hat\sigma(x)e^{-\hat\gamma(x)}},
    \frac{(y-\hat\mu(x))_+}{\hat\sigma(x)e^{\hat\gamma(x)}}
  \right\},
\]
as claimed.
\end{proof}

Proposition \ref{prop:skewgauge} shows how the family of asymmetric intervals (\ref{eq:asymm}) induces a conformity score by the dual mechanism described in Section \ref{sec:dual}. Given a pair $(X_i,Y_i)$ in the calibration sample, the corresponding conformity score is
\[
  R_i =
  \max\left\{
    \frac{(\hat\mu(X_i)-Y_i)_+}
    {\hat\sigma(X_i)e^{-\hat\gamma(X_i)}},
    \frac{(Y_i-\hat\mu(X_i))_+}
    {\hat\sigma(X_i)e^{\hat\gamma(X_i)}}
  \right\},
\]
for $i=1,\dots,n$. Proposition \ref{prop:mvpdual} yields the marginally valid prediction interval
\[
  C_{\hat r}(x)
  =
  \left[
    \hat\mu(x)-\hat r\,\hat\sigma(x)e^{-\hat\gamma(x)},
    \hat\mu(x)+\hat r\,\hat\sigma(x)e^{\hat\gamma(x)}
  \right],
\]
in which $\hat r=R_{\left(\lceil(1-\alpha)(n+1)\rceil\right)}$. The next step is to figure out how to learn the skewness function $\hat\gamma$ from the training data.

\section{Learning the skewness}\label{sec:gamma}

We now describe how the skewness function $\hat\gamma$ is learned from the training sample. The construction proceeds sequentially, in the same style as the scaled-score method of \cite{papadopoulos2002} described in Section \ref{sec:canonical}. First, the central prediction function $\hat\mu$ is learned from the training sample. Second, the absolute residuals $\Delta_i=|Y'_i-\hat\mu(X'_i)|$, for $i=1,\dots,m$, are used as targets for learning a positive predictive model $\hat\sigma:\mathbb{R}^d\to\mathbb{R}^+$. The additional step in the skew-adaptive procedure is to learn a third function $\hat\gamma:\mathbb{R}^d\to\mathbb{R}$, whose role is to determine how, for a given feature vector $x\in\mathbb{R}^d$, the predictive uncertainty should be tilted to the left or to the right of the central prediction.

After $\hat\mu$ and $\hat\sigma$ have been trained, define the signed scaled residuals
\[
  Z_i = \frac{Y'_i-\hat\mu(X'_i)}{\hat\sigma(X'_i)},
\]
for $i=1,\dots,m$. Unlike the absolute residuals $\Delta_i$, the $Z_i$'s retain the information about the direction of the prediction error. Positive values of $Z_i$ indicate that the response value lies to the right of the central prediction, while negative values indicate that it lies to the left.

We are then left with the question of how to construct a natural regression target $\tau_i$ for the skewness model $\hat\gamma$ from the signed scaled residuals $Z_i$. The construction of this target $\tau_i$ is dictated by the form of the interval family in (\ref{eq:asymm}). Since $Z_i$ is a signed scaled residual, it carries exactly the directional information that is missing from the absolute residuals $\Delta_i$ used to train $\hat\sigma$. The skewed interval family suggests that this directional information should be encoded multiplicatively rather than additively. Indeed, since the right and left sides of the skewed interval (\ref{eq:asymm}) are controlled by the reciprocal factors $e^{\hat\gamma(x)}$ and $e^{-\hat\gamma(x)}$, interchanging $Z_i\leftrightarrow-Z_i$ should interchange $e^{\hat\gamma(X'_i)}\leftrightarrow e^{-\hat\gamma(X'_i)}$.

Since the target $\tau_i$ plays, at the level of the individual training residual, the role that $\hat\gamma(X'_i)$ will play after smoothing across the feature space, we arrive at the definition of the target $\tau_i$ by imposing to it exactly the same symmetry property: interchanging $Z_i\leftrightarrow-Z_i$ interchanges $e^{\tau_i}\leftrightarrow e^{-\tau_i}$. The simplest relation with this property is $Z_i=e^{\tau_i}-e^{-\tau_i}$. Therefore, $Z_i=2\sinh(\tau_i)$, and inversion gives the target $\tau_i=\operatorname{arcsinh}(Z_i/2)$. The third predictive model $\hat\gamma$ is then trained on the pairs $\{(X'_i,\tau_i)\}_{i=1}^m$.

The full skew-adaptive conformal prediction procedure is presented in Algorithm \ref{algo:skewadaptive}. When $\hat\gamma\equiv 0$, the calibration scores in Algorithm \ref{algo:skewadaptive} reduce to $r_i=|y_i-\hat\mu(x_i)|/\hat\sigma(x_i)$, and the final interval becomes exactly the scaled-score prediction interval (\ref{eq:scaledinterval}). This fact makes the proposed method a genuine extension of the construction of \cite{papadopoulos2002}: it preserves the same central and scale components, while adding a learned skewness component that allows the two sides of the prediction interval to adapt separately.

\begin{algorithm2e}[t!]
\caption{Skew-adaptive conformal prediction}\label{algo:skewadaptive}
\SetArgSty{textbf}
\SetAlgoVlined
\DontPrintSemicolon
\LinesNumbered
\KwIn{Observed training sample $\{(x'_1,y'_1),\dots,(x'_m,y'_m)\}$ and calibration sample $\{(x_1,y_1),\dots,(x_n,y_n)\}$. Batch of future feature vectors $\{x_{n+1},\dots,x_{n+\ell}\}\subset\mathbb{R}^d$. Nominal miscoverage level $0<\alpha<1$ satisfying $\lceil(1-\alpha)(n+1)\rceil\le n$.}
\KwOut{Skew-adaptive conformal prediction intervals $C^{(1)}(x_{n+1}),\dots,C^{(\ell)}(x_{n+\ell})$.}
\BlankLine
Build $\hat\mu:\mathbb{R}^d\to\mathbb{R}$ from $\{(x'_1,y'_1),\dots,(x'_m,y'_m)\}$\;
\For{$i\leftarrow 1$ \KwTo $m$}{
  $\delta_i\gets |y'_i-\hat\mu(x'_i)|$\;
}
Build $\hat\sigma:\mathbb{R}^d\to\mathbb{R}_+$ from $\{(x'_1,\delta_1),\dots,(x'_m,\delta_m)\}$\;
\For{$i\leftarrow 1$ \KwTo $m$}{
  $z_i\gets (y'_i-\hat\mu(x'_i))/\hat\sigma(x'_i)$\;
  $\tau_i\gets \operatorname{arcsinh}(z_i/2)$\;
}
Build $\hat\gamma:\mathbb{R}^d\to\mathbb{R}$ from $\{(x'_1,\tau_1),\dots,(x'_m,\tau_m)\}$\;
\For{$i\leftarrow 1$ \KwTo $n$}{
  $r_i\gets
  \max\left\{
  \displaystyle\frac{(\hat\mu(x_i)-y_i)_+}{\hat\sigma(x_i)e^{-\hat\gamma(x_i)}},
  \displaystyle\frac{(y_i-\hat\mu(x_i))_+}{\hat\sigma(x_i)e^{\hat\gamma(x_i)}}
  \right\}$\;
}
$\hat r\gets r_{(\lceil(1-\alpha)(n+1)\rceil)}$\;
\For{$j\leftarrow 1$ \KwTo $\ell$}{
  $C^{(j)}(x_{n+j})\gets
  \left[
  \hat\mu(x_{n+j})-\hat r\,\hat\sigma(x_{n+j})e^{-\hat\gamma(x_{n+j})},
  \hat\mu(x_{n+j})+\hat r\,\hat\sigma(x_{n+j})e^{\hat\gamma(x_{n+j})}
  \right]$\;
}
\Return{$\{C^{(1)}(x_{n+1}),\dots,C^{(\ell)}(x_{n+\ell})\}$}
\end{algorithm2e}

\section{Prediction interval efficiency}\label{sec:efficiency}

In this section, we develop tools for comparing the efficiency of the prediction intervals produced by the scaled-score method of \cite{papadopoulos2002} and the skew-adaptive procedure formalized in Algorithm \ref{algo:skewadaptive}. We construct an estimator, based on the calibration sample, for the expected ratio between the widths of the intervals produced by the two methods. The functions $\hat\mu$, $\hat\sigma$, and $\hat\gamma$ have already been constructed from the training data and will be treated as deterministic throughout this section.

For the purpose of comparing the two methods, some notational adjustments relative to the preceding sections will be necessary. Objects associated with the skew-adaptive procedure will be marked with an asterisk. In both methods, the conformal calibration threshold will be denoted with an explicit index indicating the size of the calibration sample, namely
\[
  \hat r_n=R_{(\lceil(1-\alpha)(n+1)\rceil)} \quad \text{and} \quad \hat r^*_n=R^*_{(\lceil(1-\alpha)(n+1)\rceil)}.
\]

We now denote the data sequence by
\[
  (X,Y),(X_1,Y_1),(X_2,Y_2),\dots
\]
in which $(X,Y)$ represents a generic pair outside the calibration sequence $\{(X_i,Y_i)\}_{i\ge 1}$. For the following development, we impose the stronger assumption that the pairs in this sequence are independent and identically distributed (IID), rather than merely exchangeable as in the validity results of the preceding sections.

For $x\in\mathbb{R}^d$, it follows from definitions~(\ref{eq:scaledinterval}) and~(\ref{eq:asymm}) that the ratio between the widths of the intervals produced by the skew-adaptive and scaled-score methods is
\[
  \frac{|C^*_n(x)|}{|C_n(x)|}
  =
  \left(\frac{\hat r^*_n}{\hat r_n}\right)\cosh(\hat\gamma(x)).
\]
Let $H(x)=\cosh(\hat\gamma(x))$, and define the expected width ratio
\[
  \varphi_n=\mathbb{E}\!\left[\frac{|C^*_n(X)|}{|C_n(X)|}\right]=\mathbb{E}\!\left[\left(\frac{\hat r^*_n}{\hat r_n}\right)\right]\times\mathbb{E}[H(X)],
\]
in which the second equality follows from the IID assumption.

\begin{proposition}\label{prop:expectation}
Assume that $\hat r_n \to q$ a.s. and $\hat r_n^* \to q^*$ a.s., in which $q$ and $q^*$ denote the asymptotic calibration thresholds of the scaled-score and skew-adaptive procedures, respectively. If $q>0$, $\hat r_n>0$ a.s., and the sequence $\{\hat r_n^*/\hat r_n\}_{n\ge1}$ is uniformly integrable, then $\mathbb{E}[\hat r_n^*/\hat r_n]\to q^*/q$.
\end{proposition}

\begin{proof}
Since $\hat r_n\to q$ a.s. with $q>0$ and $\hat r_n^*\to q^*$ a.s., continuity of the quotient map gives $\hat r_n^*/\hat r_n\to q^*/q$ a.s. Write $U_n=\hat r_n^*/\hat r_n$ and $u=q^*/q$. We prove directly that $\mathbb{E}[|U_n-u|]\to0$. Let $\epsilon>0$. By uniform integrability, choose $M>|u|$ such that $\sup_{n\ge1}\mathbb{E}[|U_n|I_{\{|U_n|>M\}}]<\epsilon$. Then
\[
  \mathbb{E}[|U_n-u|]
  \le
  \mathbb{E}[|U_n-u|I_{\{|U_n|\le M\}}]
  +
  \mathbb{E}[|U_n-u|I_{\{|U_n|>M\}}].
\]
The first term converges to zero by dominated convergence, since $U_n\to u$ a.s. and ${|U_n-u|I_{\{|U_n|\le M\}}\le M+|u|}$. For the second term,
\[
  \mathbb{E}[|U_n-u|I_{\{|U_n|>M\}}]
  \le
  \mathbb{E}[|U_n|I_{\{|U_n|>M\}}]
  +
  |u|\,\mathbb{P}(|U_n|>M).
\]
Moreover, $\mathbb{P}(|U_n|>M)\le M^{-1}\mathbb{E}[|U_n|I_{\{|U_n|>M\}}]$, and therefore
\[
  \mathbb{E}[|U_n-u|I_{\{|U_n|>M\}}]
  \le
  \left(1+\frac{|u|}{M}\right)
  \mathbb{E}[|U_n|I_{\{|U_n|>M\}}]
  \le
  \left(1+\frac{|u|}{M}\right)\epsilon.
\]
It follows that $\limsup_{n\to\infty}\mathbb{E}[|U_n-u|]\le(1+|u|/M)\epsilon$. Since $\epsilon$ is arbitrary, $\mathbb{E}[|U_n-u|]\to0$. Finally,
\[
  \left|\mathbb{E}\!\left[\frac{\hat r_n^*}{\hat r_n}\right]-\frac{q^*}{q}\right|
  =
  |\mathbb{E}[U_n]-u|
  \le
  \mathbb{E}[|U_n-u|]
  \to0,
\]
which proves the result.
\end{proof}

In our setting, the uniform integrability of $\{\hat r_n^*/\hat r_n\}_{n\ge1}$ means that the random ratio between the skew-adaptive and scaled-score calibration thresholds cannot have rare but arbitrarily large values that remain relevant in expectation. Almost sure convergence already says that $\hat r_n^*/\hat r_n$ settles around the asymptotic ratio $q^*/q$ for large calibration samples, but this pointwise stabilization does not by itself prevent exceptional calibration samples from producing very large ratios with small probability. Uniform integrability rules out precisely this possibility: it says that the contribution of the tail event $\{\hat r_n^*/\hat r_n>M\}$ to the expected ratio can be made uniformly small, independently of $n$, by choosing $M$ large enough. Thus, in the present comparison, the uniform integrability assumption ensures that the empirical ratio of calibration thresholds is not only converging along almost every realization of the calibration sequence, but is also sufficiently well behaved in the tails for its expectation to converge to the asymptotic ratio $q^*/q$.

\begin{proposition}\label{prop:convergence}
Under the assumptions of Proposition \ref{prop:expectation}, assume that $\mathbb{E}[H(X)]<\infty$, and define the estimator
\[
  \hat\varphi_n
  =
  \left(\frac{\hat r_n^*}{\hat r_n}\right)
  \frac1n\sum_{i=1}^n H(X_i).
\]
Then $\hat\varphi_n-\varphi_n\to0$ a.s.
\end{proposition}

\begin{proof}
By adding and subtracting $\left(\hat r_n^*/\hat r_n\right)\mathbb{E}[H(X)]$, we obtain
\[
\begin{aligned}
  \hat\varphi_n-\varphi_n
  &=
  \left(\frac{\hat r_n^*}{\hat r_n}\right)
  \frac1n\sum_{i=1}^n H(X_i)
  -
  \mathbb{E}\!\left[\frac{\hat r_n^*}{\hat r_n}\right]\mathbb{E}[H(X)]
  \\
  &=
  \left(\frac{\hat r_n^*}{\hat r_n}\right)
  \left(
    \frac1n\sum_{i=1}^n H(X_i)-\mathbb{E}[H(X)]
  \right)
  +
  \left(
    \frac{\hat r_n^*}{\hat r_n}
    -
    \mathbb{E}\!\left[\frac{\hat r_n^*}{\hat r_n}\right]
  \right)
  \mathbb{E}[H(X)].
\end{aligned}
\]
The first term converges to zero a.s., since $\hat r_n^*/\hat r_n\to q^*/q$ a.s. and is therefore eventually bounded a.s., while the strong law of large numbers gives $n^{-1}\sum_{i=1}^n H(X_i)\to\mathbb{E}[H(X)]$ a.s. For the second term, Proposition \ref{prop:expectation} gives $\mathbb{E}[\hat r_n^*/\hat r_n]\to q^*/q$, whereas $\hat r_n^*/\hat r_n\to q^*/q$ a.s.; hence $\hat r_n^*/\hat r_n-\mathbb{E}[\hat r_n^*/\hat r_n]\to0$ a.s. Therefore $\hat\varphi_n-\varphi_n\to0$ a.s., and the result follows.
\end{proof}

With this construction, after the training stage has been completed, the estimator $\hat\varphi_n$ makes it possible to assess the relative efficiency of the two methods for future predictions directly from the calibration sample. More precisely, it estimates the expected ratio between the width of the skew-adaptive interval and the width of the scaled-score interval for a new observation drawn from the same distribution.

\section{Experiments}\label{sec:experiments}

We compare the performance of three split conformal prediction procedures: the scaled-score method of \cite{papadopoulos2002}, the skew-adaptive procedure developed in the paper and formalized in Algorithm \ref{algo:skewadaptive}, and the conformalized quantile regression method of \cite{romano2019}. Since all three methods share the same marginal validity, the comparison amounts to a question of prediction interval efficiency, measured by the average interval length on the test set.

We evaluate the three procedures on seven publicly available regression datasets. The selection covers a range of sample sizes and numbers of features, and includes problems for which the conditional distribution of the response variable is known to exhibit asymmetry, such as housing prices and bike rental counts. The \textit{Ames Housing} dataset \citep{decock2011} contains 2,930 observations and records sale prices of residential properties sold in Ames, Iowa, between 2006 and 2010, with eighty predictors. In the version distributed with the companion package for \cite{james2021}, the \textit{Bike Sharing} dataset \citep{fanaeet2014} contains 8,645 observations and records the hourly count of bike rentals from the Capital Bikeshare system in Washington, D.C., between 2011 and 2012, with twelve predictors. Based on the 1990 U.S. Census, the \textit{California Housing} dataset \citep{pace1997} contains 20,433 observations and reports the median house value across census tracts in California, with nine predictors. Laboratory measurements of the compressive strength of high-performance concrete mixtures are provided by the \textit{Concrete Compressive Strength} dataset \citep{yeh1998}, which contains 1,030 observations and eight quantitative inputs. The \textit{Diamonds} dataset \citep{wickham2016} contains 53,940 observations and includes prices and physical attributes of round-cut diamonds, with nine predictors. For energy-demand data, the \textit{Energy Consumption} dataset \citep{sathishkumar2021} contains 35,040 observations and reports the energy consumption of a small-scale steel manufacturing plant (Daewoo Steel Co. Ltd., Gwangyang, South Korea), recorded at 15-minute intervals over the course of one year, with ten predictors. Finally, the \textit{Used Cars} dataset \citep{najib2023} contains 4,009 observations and consists of online used-car listings, with eleven predictors. For convenience, the attributes of the seven datasets are summarized in Table \ref{tab:datasets}.

\begin{table}[t!]
\centering
\caption{Attributes of the seven datasets used in the experiments.}
\label{tab:datasets}
\small
\vspace{6pt}
\begin{tabular}{lrrlc}
\toprule
Dataset & Sample size & Number of features & Response variable & Unit \\
\midrule
Ames Housing        & $2{,}930$  & $80$ & Sale price              & USD            \\
Bike Sharing        & $8{,}645$  & $12$ & Hourly rental count     & ---            \\
California Housing  & $20{,}433$ & $9$  & Median house value      & USD            \\
Concrete            & $1{,}030$  & $8$  & Compressive strength    & MPa            \\
Diamonds            & $53{,}940$ & $9$  & Price                   & USD            \\
Energy              & $35{,}040$ & $10$ & Energy consumption      & kWh    \\
Used Cars           & $4{,}009$  & $11$ & Listed price            & USD            \\
\bottomrule
\end{tabular}
\end{table}

Each dataset is randomly partitioned into a training sample, a calibration sample, and a test sample. Two splitting ratios are considered, depending on the available sample size, and following the discussion about the impact of calibration sample sizes in \cite{marques2025a}. For the larger datasets California Housing, Diamonds, and Energy we adopt an $80\%$--$10\%$--$10\%$ split. For Ames Housing, Bike Sharing, Concrete, and Used Cars, which have smaller or moderate sample sizes, a $50\%$--$25\%$--$25\%$ split is used in order to retain a sufficiently large calibration set for learning the conformal threshold $\hat r$.

For both the scaled-score method and the skew-adaptive procedure, the central prediction $\hat\mu$, the local scale $\hat\sigma$, and the local tilt $\hat\gamma$ are learned with Random Forests \citep{breiman2001}. For conformalized quantile regression, the conditional quantile functions at levels $\alpha/2$ and $1-\alpha/2$ are learned with a Quantile Regression Forest \citep{meinshausen2006}. In all cases, the predictive models are fitted with default hyperparameters, since the goal of these experiments is not to optimize the performance of any individual method, but to compare the three conformal procedures on an equal footing, isolating the effect of the conformal procedure from the effect of hyperparameter tuning.

Table \ref{tab:data-results} reports the empirical coverage and the average prediction interval length on the test set, for the three nominal coverage levels $1-\alpha\in\{0.80, 0.85, 0.90\}$. Empirical coverage values are close to the nominal level for all three methods on all datasets, as expected from the marginal validity of the three methods. Small fluctuations in the empirical coverage below the nominal target can be explained in terms of the distribution of the empirical coverage discussed in \cite{marques2025a}. Since validity is matched across methods, the comparison reduces to prediction interval length, and on this front the skew-adaptive procedure produces, on average, the shortest intervals on every dataset and at every nominal coverage level.

Figure \ref{fig:prediction_intervals} shows, at the nominal coverage level $1-\alpha = 0.90$, the prediction intervals produced by the skew-adaptive procedure and the conformalized quantile regression method in a random sample of $30$ test observations from the Ames Housing, Bike Sharing, and California Housing datasets.

Table \ref{tab:ratio} compares, at nominal coverage level $1-\alpha=0.90$, the efficiency of the skew-adaptive method and the scaled-score method through the calibration-sample estimator $\hat\varphi_n$ introduced in Proposition \ref{prop:convergence} and the corresponding average width ratio observed on the test sample. The ratio is always defined as the width of the skew-adaptive interval divided by the width of the scaled-score interval, both for the calibration-based estimate and for the test-sample average. The agreement is remarkably close across all datasets. In five of the seven cases, the absolute difference is below $10^{-3}$, and the largest discrepancies, observed for the Bike Sharing and Used Cars datasets, remain small in practical terms. These results indicate that, after the training stage has been completed, the calibration sample can be used effectively not only to determine the conformal thresholds underlying the coverage guarantee, but also to estimate the expected relative future width of the intervals produced by the skew-adaptive and scaled-score methods. Since all ratios are below one, the table also reinforces the empirical finding that the skew-adaptive procedure produces shorter intervals than the scaled-score method on the considered datasets.

\begin{table}[t!]
\centering
\caption{Empirical coverage and average prediction interval length on the test set for the seven datasets, across three nominal coverage levels. For each dataset and each coverage level, the shortest average length is highlighted in bold.}
\label{tab:data-results}
\small
\resizebox{\textwidth}{!}{
\begin{tabular}{llcccccc}
\toprule
Dataset & $1-\alpha$
& \multicolumn{3}{c}{Empirical coverage}
& \multicolumn{3}{c}{Average length} \\
\cmidrule(lr){3-5}\cmidrule(lr){6-8}
& & Skew-adaptive & Scaled-score & CQR & Skew-adaptive & Scaled-score & CQR \\
\midrule
\multirow{3}{*}{Ames Housing}
& 90\% & 90.64\% & 90.50\% & 90.07\% & \textbf{60{,}676.05} & 62{,}525.15 & 84{,}720.87 \\
& 85\% & 86.24\% & 86.52\% & 85.11\% & \textbf{52{,}011.22} & 55{,}824.96 & 69{,}733.38 \\
& 80\% & 82.27\% & 81.28\% & 80.43\% & \textbf{45{,}369.79} & 48{,}360.39 & 60{,}602.10 \\
\midrule
\multirow{3}{*}{Bike Sharing}
& 90\% & 90.67\% & 89.55\% & 91.65\% & \textbf{110.54} & 127.10 & 206.40 \\
& 85\% & 84.93\% & 85.63\% & 86.66\% & \textbf{96.21} & 113.06 & 174.01 \\
& 80\% & 79.51\% & 80.08\% & 81.30\% & \textbf{85.06} & 102.20 & 149.00 \\
\midrule
\multirow{3}{*}{California Housing}
& 90\% & 90.47\% & 91.10\% & 89.88\% & \textbf{125{,}558.50} & 130{,}871.30 & 165{,}865.40 \\
& 85\% & 86.33\% & 86.93\% & 84.94\% & \textbf{104{,}747.20} & 111{,}895.60 & 136{,}485.20 \\
& 80\% & 83.49\% & 82.39\% & 80.20\% & \textbf{93{,}677.10} & 97{,}614.60 & 116{,}708.00 \\
\midrule
\multirow{3}{*}{Concrete}
& 90\% & 88.74\% & 88.31\% & 93.51\% & \textbf{15.03} & 16.05 & 28.54 \\
& 85\% & 84.85\% & 83.13\% & 89.18\% & \textbf{12.40} & 13.83 & 23.63 \\
& 80\% & 80.09\% & 77.06\% & 83.98\% & \textbf{10.64} & 12.33 & 20.44 \\
\midrule
\multirow{3}{*}{Diamonds}
& 90\% & 89.46\% & 90.06\% & 90.03\% & \textbf{1{,}114.36} & 1{,}151.91 & 1{,}515.00 \\
& 85\% & 85.00\% & 84.85\% & 85.01\% & \textbf{940.91} & 984.00 & 1{,}279.90 \\
& 80\% & 79.82\% & 79.73\% & 79.82\% & \textbf{818.17} & 864.00 & 1{,}105.00 \\
\midrule
\multirow{3}{*}{Energy}
& 90\% & 89.33\% & 88.73\% & 91.24\% & \textbf{2.27} & 2.81 & 6.62 \\
& 85\% & 84.23\% & 84.83\% & 86.91\% & \textbf{1.93} & 2.50 & 5.61 \\
& 80\% & 79.71\% & 79.91\% & 82.38\% & \textbf{1.69} & 2.19 & 4.86 \\
\midrule
\multirow{3}{*}{Used Cars}
& 90\% & 90.99\% & 90.37\% & 90.89\% & \textbf{52{,}109.99} & 56{,}061.27 & 75{,}702.21 \\
& 85\% & 86.02\% & 87.27\% & 86.02\% & \textbf{44{,}232.12} & 48{,}984.04 & 59{,}198.16 \\
& 80\% & 81.06\% & 80.64\% & 82.19\% & \textbf{38{,}962.15} & 43{,}208.58 & 47{,}699.34 \\
\bottomrule
\end{tabular}
}
\end{table}

\begin{figure}[t!]
\centering
\includegraphics[width=13cm]{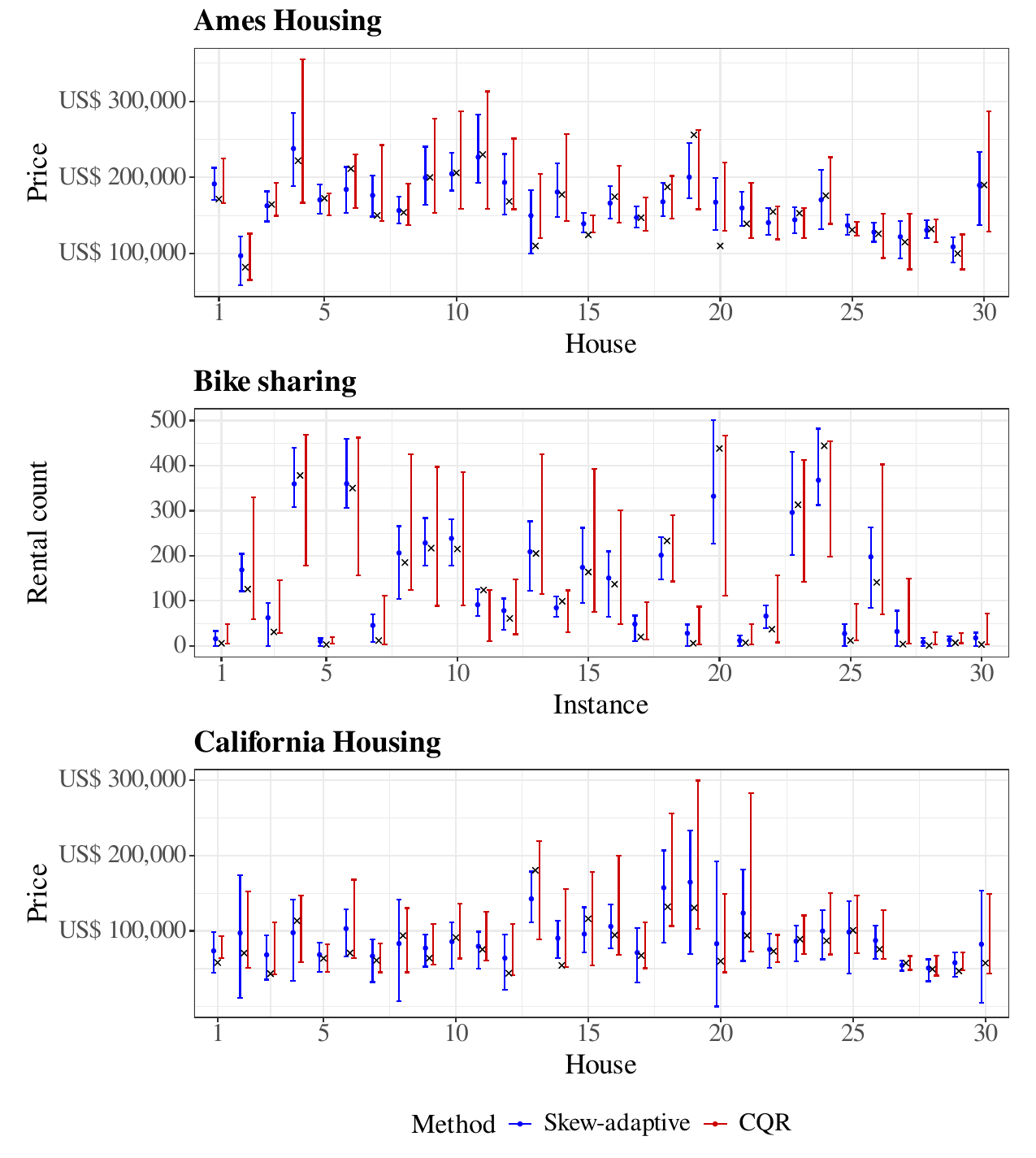}
\caption{Prediction intervals for 30 test sample units from the Ames Housing, California Housing, and Bike Sharing datasets at nominal coverage level $1-\alpha=0.90$. For each test point, the figure displays the skew-adaptive conformal prediction interval in blue, the conformalized quantile regression interval in red, and the observed response, indicated by a black cross ($\times$). The blue dot is the central prediction made by $\hat\mu$ in the skew-adaptive procedure.}
\label{fig:prediction_intervals}
\end{figure}

\begin{table}[t!]
\centering
\caption{Calibration-based estimation of relative prediction-interval efficiency at nominal coverage level $1-\alpha=0.90$. The estimator $\hat\varphi_n$ is compared with the corresponding average test-sample width ratio, defined as the width of the skew-adaptive interval divided by the width of the scaled-score interval. Ratios below one indicate that skew-adaptive intervals are shorter on average.}
\label{tab:ratio}
\small
\vspace{6pt}
\begin{tabular}{lccc}
\toprule
Dataset & Calibration estimate $\hat\varphi_n$ & Test sample average & Absolute difference \\
\midrule
Ames Housing       & 0.9534 & 0.9539 & 0.0005 \\
Bike Sharing       & 0.8660 & 0.8470 & 0.0190 \\
California Housing & 0.9548 & 0.9544 & 0.0003 \\
Concrete           & 0.9322 & 0.9316 & 0.0006 \\
Diamonds           & 0.9692 & 0.9686 & 0.0006 \\
Energy             & 0.8204 & 0.8200 & 0.0004 \\
Used Cars          & 0.9456 & 0.9365 & 0.0092 \\
\bottomrule
\end{tabular}
\end{table}

\section{Concluding remarks}\label{sec:conclusion}

Our development has been confined to the inductive, or ``split'', case of conformal prediction. Nevertheless, the sequential procedure developed here, which introduces the third tilt model $\hat\gamma$, is also suitable in other settings. For example, in applications of the method introduced by \cite{johansson2014}, one usually constructs models analogous to $\hat\mu$ and $\hat\sigma$ from out-of-bag predictions on the training sample, thereby eliminating the need for a separate calibration sample; see, for example, \cite{graziadei2025}. Hence, the same $\hat\mu\text{--}\hat\sigma\text{--}\hat\gamma$ scheme can be used to augment this out-of-bag procedure. An analogous possibility arises for stacked conformal prediction \citep{marques2025b}, another calibration-sample-free method, in which the conformalization of the meta-learner could be enriched by the inclusion of a $\hat\gamma$-like model.

Since the seminal work of \cite{vovk2005}, the development of conformal prediction has aimed at the construction of finite-sample statistical guarantees, a goal that is inherited by the skew-adaptive procedure. On the complementary side of this pursuit of finite-sample properties, Section \ref{sec:efficiency} explores a connection between the usual constructs of conformal prediction and the concepts and methods of classical asymptotic theory. The possibility of empirically comparing the efficiency of different conformal prediction methods by applying asymptotic machinery to the calibration sample opens interesting possibilities and, in our view, is a manifestation of the recurrent fact that in Statistical Theory all its parts insist on being connected. Reference implementations of the skew-adaptive conformal prediction procedure, written in R \citep{R}, are available at:
\begin{center}
\texttt{https://github.com/paulocmarquesf/skew-adaptive\_cp} \quad
\end{center}

\acks{Paulo C. Marques F. receives support from FAPESP (Fundação de Amparo à Pesquisa do Estado de São Paulo) through project 2023/02538-0.}

\bibliography{bibliography.bib}

\end{document}